\documentclass[twoside]{article}

\usepackage[accepted]{aistats2026}
\usepackage[round]{natbib}

\usepackage[utf8]{inputenc} % allow utf-8 input
\usepackage[T1]{fontenc}    % use 8-bit T1 fonts
\usepackage{hyperref}       % hyperlinks
\usepackage{url}            % simple URL typesetting
\usepackage{booktabs}       % professional-quality tables
\usepackage{amsfonts}       % blackboard math symbols
\usepackage{nicefrac}       % compact symbols for 1/2, etc.
\usepackage{microtype}      % microtypography
\usepackage{xcolor}         % colors
\usepackage{algcompatible}
\usepackage{wrapfig}

\hypersetup{
    colorlinks=true,
    linkcolor=blue,
    citecolor=blue,
    urlcolor=blue,
}

\usepackage{amssymb, amsmath, amsthm}
\usepackage{mathtools}
\usepackage{dsfont}
\usepackage{mathrsfs}

\usepackage{lipsum}  

\usepackage{pgffor}

\newtheorem{theorem}{Theorem}
\newtheorem{lemma}[theorem]{Lemma}

\newtheorem{remark}[theorem]{Remark}
\newtheorem{algorithm}[theorem]{Algorithm}

\begin{document}



\twocolumn[

\aistatstitle{Lloyd's $K$-Means Clustering Algorithm Is Frank-Wolfe in Disguise}

\aistatsauthor{Michael Pokojovy \And J.~Marcus Jobe \And  Simon Lacoste-Julien}

\aistatsaddress{
  Department of Mathematics \& Statistics \\ 
  and School of Data Science \\
  Old Dominion University \\
  Norfolk, VA 23529, USA \\
  \texttt{mpokojovy@odu.edu}
  \And
  Information Systems \& \\
  Analytics Department \\
  Miami University \\
  Oxford, OH 45056, USA \\
  \texttt{jobejm@miamioh.edu}
  \And
  Canada CIFAR AI Chair \\
  Department of Computer Science and \\
  Operations Research \& Mila \\
  Universit\'{e} de Montr\'{e}al \\
  Montr\'{e}al (QC) Canada}]

\begin{abstract}
    Lloyd's $K$-means algorithm, also known as na\"{i}ve $K$-means, is a widely used {\it ad hoc} optimization heuristic, designed to minimize the sum of squared errors (SSE) across all $K$-partitions of a dataset via iterative cluster refinement. In this work, we establish a novel connection between Lloyd's algorithm and the Frank-Wolfe (FW) algorithm, a prominent first-order method for projection-free optimization. We demonstrate that Lloyd's algorithm is a special case of FW. Leveraging recent advances in FW methods for concave objectives, we derive a non-asymptotic $\mathcal{O}(1/t)$ convergence rate to a local minimum of the SSE objective. To account for empty clusters, an outcome possible under Lloyd's greedy assignment, we develop an FW variant for semismooth objectives while retaining the same convergence rate that is solely controlled by the initial SSE value. We illustrate our findings with a simulation study for spherical Gaussian mixtures and a real-world image segmentation dataset.
\end{abstract}

\section{\MakeUppercase{Introduction}}

Cluster analysis is a branch of unsupervised machine learning focused on partitioning unlabeled data into meaningful groups, or clusters, based on similarity or other relevant criteria. $K$-means clustering is one of the most popular and widely used clustering approaches. While $K$-means clustering is an NP-hard optimization problem~\citep{arthur2006how}, a variety of heuristic solution approaches, most notably Lloyd's algorithm~\citep{Llo1982}, exist. While initialization and acceleration strategies can be quite nuanced and vary from algorithm to algorithm~\citep{hamerly2010making}, the shared core of most implementations is the iterative reassignment strategy of~\cite{Llo1982}. Regarded as one of the Top 10 algorithms in data mining~\citep{wu2008top}, ``the'' $K$-means algorithm has been extensively studied in the literature and is nearly ubiquitous across various application fields. See~\cite{blomer2016theoretical, ikotun2023kmeans, steinley2006kmeans} for thorough review and synthesis of notable results.

As a multi-faceted topic, $K$-means clustering has attracted significant attention from diverse communities including machine learning, data mining, computer science, statistics, optimization, etc. In this work, we are specifically interested in optimization aspects. A remarkable property of the minimization problem underlying $K$-means clustering is that it can be equivalently formulated using the language of combinatorial programming, mixed integer nonlinear programming, expectation maximization (EM), smooth and nonsmooth optimization, matrix factorization, etc.~\citep{bagirov2015nonsmooth, bauckhage2016kmeans, bottou1995convergence}. These paradigms offer unique insights into the nature of the problem and provide instruments of theoretical and practical importance. Continuing this promising pursuit, the present paper reveals another notable connection: {\it Lloyd's $K$-means algorithm is a special case of Frank-Wolfe (FW) algorithm}. In addition to mathematical elegance, this fact offers a unique opportunity to leverage the techniques of projection-free concave optimization in studying the convergence of Lloyd's $K$-means algorithm. This strategy has recently been successfully applied to analyzing Gaussian trimmed likelihood estimation~\citep{PoJo2022} and the convex-concave procedure~\citep{yurtsever2022cccp} through the lens of FW algorithm.

Motivated by the semismooth concave nature of the $K$-means objective, we leveraged,  adapted and improved some recent advancements in the field~\citep{khamaru2019convergence, yurtsever2022cccp} to address the challenges posed by Lloyd's $K$-means algorithm. The latter are specifically caused by empty clusters oftentimes resulting from Lloyd's greedy allocation that necessitates semismooth extension of otherwise smooth objective to an appropriate convex set.

\paragraph{Contributions.}
Main contributions of this work can be briefly summarized as follows:
\begin{itemize}
    \item We show that the Lloyd's algorithm with greedy cluster assignment is a special case of the FW algorithm with unit step-size for a semismooth concave objective over a polyhedral set. With empty clusters excluded from the feasible set, the objective becomes smooth. This observation allows us to leverage recent convergence results for concave FW to directly establish a non-asymptotic $\mathcal{O}(1/t)$ convergence rate~\citep{yurtsever2022cccp}.

    \item To accommodate for the practically relevant possibility of empty clusters, we develop an FW algorithm for general (semismooth) concave objectives with a new FW gap based on Clarke subdifferential. The same non-asymptotic $\mathcal{O}(1/t)$ convergence rate is recovered that solely depends on the initial global suboptimality (uniformly in sample size $n$, space dimension $d$ and number of clusters $K$) and does not involve coreset considerations.
\end{itemize}

In addition to self-contained proofs, a simulation study was performed to corroborate these results.

\paragraph{Relevance.}

Due to its importance, the $K$-means problem has been extensively studied both in terms of algorithmic complexity and convergence speed of existing computational heuristics~\citep{blomer2016theoretical} and remains an active research field. The ability to view Lloyd's $K$-means algorithm through the lens of semismooth FW algorithm gives a new simple, yet powerful tool that provides both convergence guarantees and a practical error control mechanism. Outside of core domain, the results are expected to have implications for more general forms of robust cluster analysis aimed to account for the presence of outliers and other model violations~\citep{dorabiala2022robust, garcia2008general}. With limited amount of convergence results available to date, primarily due to unique combinatorial challenges associated with computing robust estimators~\citep{bernholt2004complexity}, our new approach promises to have even more significance. Lastly, our work is potentially relevant for any situation, where greedy variants of the linear minimization oracle (LMO) can destroy smoothness.

\paragraph{Setup.} The $K$-means SSE objective reads as
\begin{equation}
    \label{EQUATION:K_MEANS_PROBLEM_INTRO}
    f(\boldsymbol{w}) = \sum_{k = 1}^{K} \sum_{i = 1}^{n} w_{ik} \|\boldsymbol{x}_{i} - \boldsymbol{x}_{k}^{\boldsymbol{w}}\|^{2}
\end{equation}
with $\boldsymbol{x}_{k}^{\boldsymbol{w}} := \frac{1}{n_{k}^{\boldsymbol{w}}} \sum\limits_{i = 1}^{n} w_{ik} \boldsymbol{x}_{i}$ and $n_{k}^{\boldsymbol{w}} := \sum\limits_{i =  1}^{n} w_{ik}$ using the usual `one-hot' encoding $\boldsymbol{w} \in \{0, 1\}^{n \times K}$ such that $w_{ik} = 1$ if the $i$-th data point $\boldsymbol{x}_{i} \in \mathbb{R}^{d}$ is in the $k$-th cluster. The goal is to minimize $f$ over all partitions $\boldsymbol{w} \in \mathcal{M}_{0}^{\mathrm{adm}}$ with non-overlapping non-empty clusters. Since $f$ is concave (cf.~Lemma~\ref{LEMMA:PROPERTY_OF_K_MEANS_OBJECTIVE}), the latter is equivalent to solving
\begin{equation} 
    \label{EQUATION:PROBLEM_INTRO_ADMISSIBLE_M}
    \min_{\boldsymbol{w} \in \mathcal{M}^{\mathrm{adm}}} f(\boldsymbol{w})
\end{equation}
over the convex set $\mathcal{M}^{\mathrm{adm}} := \operatorname{conv}(\mathcal{M}_{0}^{\mathrm{adm}})$~\citep{Ho1984}. Since Lloyd's $K$-means cluster assignment corresponds to FW stepping with a unit step-size in the direction provided by a greedy linear minimization oracle (LMO) over a larger convex set $\mathcal{M}$
\begin{equation} 
    \label{EQUATION:PROBLEM_INTRO}
    \min_{\boldsymbol{w} \in \mathcal{M}} f(\boldsymbol{w}),
\end{equation}
in which the objective $f$ is still concave, but only semismooth, the latter configuration is adopted and analyzed in this paper.

\section{\MakeUppercase{Related Work}}

\paragraph{Frank-Wolfe Algorithm.}

Originally proposed by~\citet{Frank:1956vp}, the FW algorithm, also known as conditional gradient descent, has recently experienced a major resurgence in interest in the machine learning community due to wide applicability and suitability for large-scale machine learning problems, adversarial learning, sparse estimation in multiple linear regression, support vector machine training, matrix completion, computational geometry, etc. See~\citep{jaggi2013revisiting, Po2023, BoRiZe2021} for applications and introduction to projection-free optimization with FW.

The FW algorithm (cf.~Algorithm~\ref{ALGORITHM:FW_UNIT_SS}) aims to solve optimization problem
\begin{equation} 
    \label{EQUATION:PROBLEM_FW}
    \min_{\boldsymbol{x} \in \mathcal{M}} f(\boldsymbol{x})
\end{equation}
for a convex compact $\emptyset \neq \mathcal{M} \subset \mathbb{R}^{d}$.
Unlike projection-based methods~\citep{nesterov2018lectures}, orthogonal projectors are replaced with linear minimization oracle (LMO)
\begin{equation}
    \label{EQUATION:LMO_INTRODUCTION}
    \text{LMO}_{\mathcal{M}}(\boldsymbol{r}) := \mathop{\operatorname{arg\,min}}_{\boldsymbol{s} \in \mathcal{M}} \langle \boldsymbol{s}, \boldsymbol{r}\rangle.
\end{equation}
Classical stepping mechanisms include the original fixed step-size scheme of~\citep{Frank:1956vp}, line search or adaptive stepping subject to a curvature constant bound~\citep{LaJu2016}. See~\citep{Bomze2019, BoRiZe2021, locatello2017unified, negiar2020stochastic} for alternative strategies. 
One of the key quantities in the ``smooth'' FW algorithm is the FW gap $g_{t} \coloneqq g(\boldsymbol{x}^{(t)})$ with
\begin{equation} 
    \label{EQUATION:FW_GAP_T}
    g(\boldsymbol{x}) \equiv \max_{\boldsymbol{s} \in \mathcal{M}} \big\langle \boldsymbol{s} - \boldsymbol{x}, -\nabla f(\boldsymbol{x})\big\rangle \quad \text{at the $t$-th iterate}.
\end{equation}
The FW gap $g(\boldsymbol{x})$ is a special case of the duality gap and, thus, can be used to gauge stationarity.

Various assumptions on the feasible set $\mathcal{M}$ and the objective $f$ are encountered in the literature. While $\mathcal{M}$ is usually assumed as a generic non-empty convex compact set---although uniform convex sets~\citep{GaHa2015} or polytopes~\citep{Guelat1986AwayStep} are sometimes specifically considered---the requirements on $f$ may vary depending on the stepping strategy, desired convergence rate and other factors. Most of the existing research focuses on convex or strongly convex smooth objectives~\citep{GaHa2015, GaHa2016, lacoste2015global, pmlr-v108-pedregosa20a}. Some results vary depending on whether the optimum $\boldsymbol{x}^{\ast}$ is a boundary or an interior point.

While there has been growing interest in applying the Frank-Wolfe method to non-convex optimization, primarily due to new applications in multiple sequence alignment~\citep[Appendix~B]{Alayrac16unsupervised}, multi-object tracking~\citep[Section~5.1]{chari15pairwise}, robust estimation of Gaussian models~\citep{PoJo2022}, fewer results exist compared to the convex case. While earlier works (e.g.,~\cite{dunn1979rates}) assumed invexity, general non-convex objectives with finite curvature have recently been studied~\citep{LaJu2016, pmlr-v108-pedregosa20a, Reddi2016Allerton} leading to $\min\limits^{}_{0 \leq \tau \leq t} g_{\tau} = \mathcal{O}(1/\sqrt{t})$ (or slower rates for less smooth functions~\citep{deOl2023}) under various step control strategies. The rate is generally inferior to the usual convex rate of $\mathcal{O}(1/t)$. See \cite[Table~2]{BoRiZe2021}, \cite[Table~1]{GaHa2015} and \cite[Table~1]{pmlr-v108-pedregosa20a} for summary. Some other works specifically focus on concave objectives~\citep{Cl2010, mangasarian1996machine, rinaldi2008concave} with a recently established $\mathcal{O}(1/t)$ convergence rate~\citep{yurtsever2022cccp} under the greedy unit step-size choice. Another important research avenue is FW algorithm for nonsmooth functions~\citep{KrishnanNIPS2015, khamaru2019convergence, ravi2019deterministic}, including for concave objectives~\citep{white1993extension} like the one we investigate in Section~\ref{SECTION:FW_ALGORITHM}. Our work closes a gap in the existing literature by developing an FW algorithm for concave semismooth objectives while relying on Clarke subdifferential in lieu of the less convenient Goldstein's subgradient and preserving recent $\mathcal{O}(1/t)$ convergence results in the smooth case.

\paragraph{$K$-means Clustering as Continuous Optimization.}

$K$-means clustering has been extensively studied from the viewpoint of continuous optimization~\citep{bagirov2015nonsmooth, bottou1995convergence, selim1984kmeans}. Unrelated to Lloyd's algorithm, \cite{bauckhage2016kmeans} proposed to use FW with diminishing step as an optimization heuristic for $K$-means. Another major connection follows from the theory of $\varepsilon$-coresets~\citep{Cl2010, harpeled2007smaller, ravi2019deterministic}. In addition to being quite involved, estimates derived for coresets, although typically bounded in the sample size $n$ and dimension $d$, tend to explode as $\varepsilon \to 0$. The novelty of our work is the ability to derive an $\mathcal{O}(1/t)$ convergence rate that solely depends on the initial suboptimality and holds uniformly in sample size $n$, dimension $d$ and number of clusters $K$. Leveraging the properties of FW algorithm in general convex geometries, no polytope structure or combinatorial arguments are involved which makes our approach robust to modifications and extensions, which recently proved important in other areas~\citep{PoJo2022, yurtsever2022cccp}.

\section{\MakeUppercase{FW for Semismooth Concave Objectives}}
\label{SECTION:FW_ALGORITHM}

Consider the optimization problem in Equation~\eqref{EQUATION:PROBLEM_FW}. %\eqref{EQUATION:PROBLEM_INTRO}.
Assume that $\mathcal{M} \subset \mathbb{R}^{d}$ is a non-empty compact convex set and $f \colon \mathcal{M} \to \mathbb{R}$ is a concave continuous function---not necessarily continuously differentiable in $\mathcal{M}$. The subdifferential $\partial f$ of $f$ can be defined as the negative of the subdifferential of the convex function $-f$. Moreover, this definition agrees with other three standard concepts of limiting, Frech\'{e}t and Clarke subdifferentials~\cite[Chapters~8-9]{rockafellar1998variational}. With a slight notation abuse, for simplicitly, we will denote all three subdifferentials by $\partial f$. Further, the set $\partial f(\boldsymbol{x}) \neq \emptyset$ is compact and convex for any $\boldsymbol{x} \in \mathcal{M}$ and, if $f$ is differentiable at $\boldsymbol{x}$, $\partial f(\boldsymbol{x}) = \{\nabla f(\boldsymbol{x})\}$ is a singleton.

FW algorithm for semismooth functions has recently been proposed and investigated by~\cite{khamaru2019convergence}. Under the umbrella of semismooth functions, various not always equivalent alternative definitions exist~\citep{rockafellar1998variational} with the shared goal of generalizing the classical subdifferential of a convex function to more general classes. Introducing a curvature constant akin to~\cite{LaJu2016} but using Frech\'{e}t subdifferential, \cite{rockafellar1998variational} showed an $\mathcal{O}(1/\sqrt{t})$ convergence of FW with variable step-sizes without rigorously explaining the stationarity concept adopted, which is provided in this paper. Additionally, assuming the objective is concave, we show an improved $\mathcal{O}(1/t)$ convergence rate by adapting the proof of~\cite{yurtsever2022cccp} presented for smooth objectives. Algorithm~\ref{ALGORITHM:FW_UNIT_SS} protocols our semismooth FW with unit step-size. Note that for smooth objectives, it automatically reduces to the smooth version since $\partial f(\boldsymbol{x}) = \{\nabla f(\boldsymbol{x})\}$.

\begin{algorithm}[Semismooth FW for concave functions\label{ALGORITHM:FW_UNIT_SS}]
            $ $\\[-0.3in]
            \begin{algorithmic}[1]
                \STATE Let $\boldsymbol{x}^{(0)} \in \mathcal{M}$.
                \FOR{$t = 0, \ldots, T - 1$}
                    \STATE Choose arbitrary $\boldsymbol{r}^{(t)} \in \partial f(\boldsymbol{x}^{(t)})$.
                    \STATE Compute $\boldsymbol{s}^{(t)} := 
                    \operatorname{LMO}_{\mathcal{M}}\big(\boldsymbol{r}^{(t)}\big)$.
                    \label{LINE:LMO}
                    \STATE Define the FW update direction $\boldsymbol{d}_{t} := \boldsymbol{s}^{(t)} - \boldsymbol{x}^{(t)}$.
                    \STATE Compute the FW gap $g_{t} := \left\langle \boldsymbol{d}_t, -\boldsymbol{r}^{(t)}\right\rangle$. \label{LINE:FW_GAP}
                    \IF{$g_{t} \leq \epsilon$}
                        \STATE \textbf{return} $\boldsymbol{x}^{(t)}$
                    \ENDIF
                    \STATE Select step-size $\gamma_{t} = 1$. \label{LINE:UNIT_SS}
                    \STATE Update $\boldsymbol{x}^{(t + 1)} := \boldsymbol{x}^{(t)} + \gamma_t \boldsymbol{d}_t \equiv \boldsymbol{s}^{(t)}$.
                \ENDFOR
                \STATE \textbf{return} $\boldsymbol{x}^{(T)}$.
            \end{algorithmic}
\end{algorithm}

Before discussing non-asymptotic convergence of Algorithm~\ref{ALGORITHM:FW_UNIT_SS}, we briefly introduce a suitable generalization of the FW gap to be used as a stationarity measure. Instead of relying on $(\delta, \epsilon)$-Goldstein stationarity~\citep{liu2024zeroth, ravi2019deterministic}, we employ the concept of Clarke stationarity. Recall that a point $\boldsymbol{x} \in \mathcal{M}$ is referred to as Clarke-stationary~\citep{deOl2023} if
\begin{equation}
    \label{EQUATION:CLARKE_STATIONARITY}
    0 = \max_{\boldsymbol{y} \in \mathcal{M}} \langle \boldsymbol{y} - \boldsymbol{x}, -\boldsymbol{r}\rangle = 0
\end{equation}
for at least one $\boldsymbol{r} \in \partial f(\boldsymbol{x})$.

Similar to~\cite{liu2024zeroth}, we introduce the Clarke version of FW gap
\begin{equation}
    g_{\mathrm{C}}(\boldsymbol{x}) := \min_{\boldsymbol{r} \in \partial f(\boldsymbol{x})} \max_{\boldsymbol{y} \in \mathcal{M}} \langle \boldsymbol{y} - \boldsymbol{x}, -\boldsymbol{r}\rangle.
\end{equation}
By compactness, the latter expression is well-defined and finite. Further, for any $\boldsymbol{r} \in \partial f(\boldsymbol{x})$,
\begin{equation*}
    \max_{\boldsymbol{y} \in \mathcal{M}} \langle \boldsymbol{y} - \boldsymbol{x}, -\boldsymbol{r}\rangle \geq \langle\boldsymbol{y} - \boldsymbol{x}, -\boldsymbol{r}\rangle|_{\boldsymbol{y} = \boldsymbol{x}} = 0.
\end{equation*}
Thus, we trivially have $g_{\mathrm{C}}(\boldsymbol{x}) \geq 0$. By Equation~\eqref{EQUATION:CLARKE_STATIONARITY}, $g_{\mathrm{C}}(\boldsymbol{x}) = 0$ if and only if $\boldsymbol{x}$ is Clarke stationary.

As for smooth functions~\citep{yurtsever2022cccp}, Lemma~\ref{LEMMA:UNIT_STEP_SIZE} shows the usual line search in FW is equivalent with the greedy choice of a unit step-size. This offers a major computational advantage since no additional function evaluations are required to solve the directional minimization problem.

\begin{lemma}
    \label{LEMMA:UNIT_STEP_SIZE}
    The unit step-size $\gamma_{t} = 1$ solves the line search problem
    \begin{equation*}
        \min_{\gamma \in [0, 1]} f\big(\boldsymbol{x}^{(t)} + \gamma \boldsymbol{d}_{t}\big).
    \end{equation*}
\end{lemma}

\begin{proof}
    Since $f$ is a concave function, so is $\varphi_{t} \colon [0, 1] \to \mathbb{R}$, $\varphi_{t}(\gamma) := f\big(\boldsymbol{x}^{(t)} + \gamma \boldsymbol{d}_{t}\big)$ at each step $t$ of Algorithm~\ref{ALGORITHM:FW_UNIT_SS}. Therefore, without necessarily being smooth, $\varphi_{t}$ must attain its global (possibly non-strict) minima at the boundary $\{0, 1\}$~\citep{Ho1984}.
    Arguing as in Equation~\eqref{EQUATION:INDUCTION_STEP} below and recalling that $g_{t} \geq 0$, we have $\varphi_{t}(1) \leq \varphi_{t}(0)$, whence $\varphi_{t}$ has a minimum at $\gamma = 1$.
\end{proof}

\begin{theorem}[Convergence of semismooth FW on concave objectives]
    \label{THEOREM:CONVERGENCE}
    
    Consider running the FW Algorithm~\ref{ALGORITHM:FW_UNIT_SS} with unit step-size for the optimization problem in Equation~\eqref{EQUATION:PROBLEM_FW}.
    Then the minimal FW gap $\tilde{g}_t := \displaystyle \min_{0 \leq \tau \leq t} g_\tau$ encountered by the iterates after $t$ iterations satisfies:
    \begin{equation} 
        \label{EQUATION:BOUND_CONCAVE}
        \tilde{g}_t \leq \frac{h_0}{t+1}  \quad \text{for $0 \leq t \leq T - 1$}
    \end{equation}
    where $h_0 := f(\boldsymbol{x}^{(0)}) - \displaystyle \min_{\boldsymbol{x} \in \mathcal{M}} f(\boldsymbol{x})$ is the initial global suboptimality.
\end{theorem}

The proof follows the streamlines of~\cite{yurtsever2022cccp} in the smooth case (see also~\cite{LaJu2016}). At each step of the algorithm, the objective is decreased by (at least) the FW gap $g_{t}$. As cumulative reduction is bounded by the initial global suboptimality over $\mathcal{M}$ (or, equivalently, $\partial \mathcal{M}$), the FW $g_{t}$ must eventually become small. Compared to the general non-convex case, the concavity of $f$ not only furnishes an improved rate of $\mathcal{O}(1/t)$ in lieu of $\mathcal{O}(1/\sqrt{t})$ but produces a bound that solely depends on the initial suboptimality.

\begin{proof}
    For arbitrary $\gamma \in [0, 1]$, compute the point $\boldsymbol{x}_{\gamma} :=  \boldsymbol{x}^{(t)} + \gamma \boldsymbol{d}^{(t)}$ by moving with
    step-size $\gamma$ in direction $\boldsymbol{s}^{(t)}$, where $\boldsymbol{d}_t := \boldsymbol{s}^{(t)} - \boldsymbol{x}^{(t)}$ is the FW direction as defined
    by Algorithm~\ref{ALGORITHM:FW_UNIT_SS}. By concavity,
    \begin{equation}
        \label{EQUATION:DESCENT_LEMMA}
        f(\boldsymbol{x}_\gamma) \leq f(\boldsymbol{x}^{(t)}) + \gamma 
        \langle \boldsymbol{d}_{t}, \boldsymbol{r}\rangle
        \text{ for all } \boldsymbol{r} \in \partial f(\boldsymbol{x}^{(t)}).
    \end{equation}
    Letting $\gamma = 1$ and $\boldsymbol{r} = \boldsymbol{r}^{(t)}$, we get
    \begin{align}
        \begin{split}
            \label{EQUATION:INDUCTION_STEP}
            f(\boldsymbol{s}^{(t)}) 
            &\leq
            f(\boldsymbol{x}^{(t)}) - \langle \boldsymbol{d}_{t}, -\boldsymbol{r}^{(t)}\rangle
            = f(\boldsymbol{x}^{(t)}) - g_{t}
        \end{split}
    \end{align}
    for $t = 0, \dots, T - 1$.
    By induction,
    \begin{equation}
        \label{EQUATION:INDUCTION_STEP_ITERATED}
        f\big(\boldsymbol{x}^{(t + 1)}\big) \leq f\big(\boldsymbol{x}^{(0)}\big) - \sum_{\tau = 0}^{t} g_{\tau}
    \end{equation}
    for $t = 0, \dots, T - 1$ and, therefore,
    \begin{equation}
        \label{EQUATION:INDUCTION_RESOLVED}
        \begin{split}
            f\big(\boldsymbol{x}^{(t + 1)}\big) &\leq f\big(\boldsymbol{x}^{(0)}\big) - (t + 1) \min_{0 \leq \tau \leq t} g_{\tau} \\
            &=
            f\big(\boldsymbol{x}^{(0)}\big) - (t + 1) \tilde{g}_t \text{ for } t = 0, \dots, T
        \end{split}
    \end{equation}
    with $\tilde{g}_t := \displaystyle \min_{0 \leq \tau \leq t} g_\tau$.
    Estimating for $t = 0, \dots, T - 1$
    \begin{equation*}
        f(\boldsymbol{x}^{(0)}) - f(\boldsymbol{x}^{(t+1)}) \leq f(\boldsymbol{x}^{(0)}) - \min_{\boldsymbol{x} \in \mathcal{M}} f(\boldsymbol{x}) \equiv h_0
    \end{equation*}
    and solving Equation~\eqref{EQUATION:INDUCTION_RESOLVED} for $\tilde{g}_t$, the claim follows.
\end{proof}

\begin{remark}
    In case of ambiguity, the LMO in FW Algorithm~\ref{ALGORITHM:FW_UNIT_SS}, line~\ref{LINE:LMO} can be amended to always return an extreme point of $\mathcal{M}$~\citep{Ho1984}. This property is trivially inherited by the iterates $\boldsymbol{x}^{(t)}$ for $t = 1, \dots, T$. Also, if $\mathcal{M}$ is a convex polytope, i.e., $\mathcal{M} = \operatorname{conv}(\mathcal{M}_{0})$ for some finite $\mathcal{M}_{0} \neq \emptyset$, one can similarly argue the FW Algorithm~\ref{ALGORITHM:FW_UNIT_SS} must terminate in $|\mathcal{M}_{0}|$ steps or less.
\end{remark}

%%%%%%%%%%%%%%%%%%%%%%%%%%%%%%%%%%%%%%%%%%%%%%%%%%%%%%%%%%%%%%%%%
% K-Means
%%%%%%%%%%%%%%%%%%%%%%%%%%%%%%%%%%%%%%%%%%%%%%%%%%%%%%%%%%%%%%%%%

\section{\MakeUppercase{Lloyd's $K$-Means as FW for Concave Objective}}
\label{SECTION:K_MEANS_CLUSTERING}

We apply the general results of Section~\ref{SECTION:FW_ALGORITHM} to the $K$-means clustering. Starting with the case of non-empty clusters, Lemma~\ref{LEMMA:PROPERTY_OF_K_MEANS_OBJECTIVE} below characterizes the SSE objective as a smooth concave function on a compact set, directly producing an $\mathcal{O}(1/t)$ convergence in the FW gap. In case of empty cluster(s), the argumentation is somewhat more involved. For any cluster turning empty at some iteration $t$, an ``ephemerous'' center can selected using the previous center, performing random sampling from the dataset or using another strategy. We show that any such selection strategy corresponds to selecting an element of the Clarke subdifferential of the extended semismooth concave objective, which, in turn, yields the same $\mathcal{O}(1/t)$ convergence rate.

Consider a dataset $\mathcal{D} = \{\boldsymbol{x}_{1}, \boldsymbol{x}_{2}, \dots, \boldsymbol{x}_{n}\} \subset \mathbb{R}^{d}$. We seek to minimize the sum of squares errors (SSE), also known as within-cluster sum of squares (WCSS),
\begin{equation}
    \label{EQUATION:WCSS_OBJECTIVE_EQUIVALENT}
    \mathrm{SSE}(\mathcal{C}) \equiv \sum_{k: C_{k} \neq \emptyset} 
    \sum_{\boldsymbol{x}\in C_{k}} \big\|\boldsymbol{x} - \bar{\boldsymbol{x}}_{k}\|^{2}
\end{equation}
with $\bar{\boldsymbol{x}}_{k} = \frac{1}{|C_{k}|} \sum\limits^{}_{\boldsymbol{x} \in C_{k}} \boldsymbol{x}$
over all $K$-partitions $\mathcal{C} = (C_{1}, C_{2}, \dots, C_{K})$ of $\mathcal{D}$ of size $K \in \mathbb{N}$. 

\begin{algorithm}[Lloyd's $K$-means algorithm\label{ALGORITHM:LLOYD}]
            $ $
            \begin{algorithmic}[1]
                \STATE Let the seeds $\boldsymbol{\mu}_{1}^{(-1)}, \dots, \boldsymbol{\mu}_{K}^{(-1)} \in \mathbb{R}^{d}$ be given.

                \STATE Initialize: \\
                $C_{k}^{(0)} := \big\{\boldsymbol{x}_{i} \,|\, \forall k': \|\boldsymbol{x}_{i} - \boldsymbol{\mu}_{k}^{(-1)}\| \leq \|\boldsymbol{x}_{i} - \boldsymbol{\mu}_{k'}^{(-1)}\|\big\}$, \\
                $\boldsymbol{\mu}_{k}^{(0)} := 
                \begin{cases}
                    \frac{1}{|C_{k}^{(0)}|} \sum\limits^{}_{\boldsymbol{x}_{i} \in C_{k}^{(0)}} \boldsymbol{x}_{i}, & \text{if } C^{(0)}_{k} \neq \emptyset, \\
                    \boldsymbol{\mu}_{k}^{(0)} := \boldsymbol{\mu}_{k}^{(-1)}, &\text{otherwise}.
                \end{cases}$
                
                \FOR{$t = 0, \ldots, T - 1$}
                \FOR{$k = 1, \ldots, K$}
                    \STATE Let \\
                    $C_{k}^{(t+1)} := \big\{\boldsymbol{x}_{i} \,|\, \forall k': \|\boldsymbol{x}_{i} - \boldsymbol{\mu}_{k}^{(t)}\| \leq \|\boldsymbol{x}_{i} - \boldsymbol{\mu}_{k'}^{(t)}\|\big\}$. \label{LINE:LLOYD_GREEDY_ASSIGNMENT}

                    \STATE Compute \\
                    $\boldsymbol{\mu}_{k}^{(t+1)} := 
                    \begin{cases}
                        \frac{1}{|C_{k}^{(t+1)}|} \sum\limits^{}_{\boldsymbol{x}_{i} \in C_{k}^{(t+1)}} \boldsymbol{x}_{i}, & \text{if } C_{k}^{(t+1)} \neq \emptyset, \\
                        \boldsymbol{\mu}_{k}^{(t)}, &\text{otherwise}.
                    \end{cases}$

                    \STATE Let $\Delta \mathrm{SSE}_{t} := \mathrm{SSE}(\mathcal{C}^{(t)}) - \mathrm{SSE}(\mathcal{C}^{(t+1)})$.
                    
                    \IF{$\Delta \mathrm{SSE}_{t} \leq \epsilon$}
                        \STATE \textbf{return} $C_{1}^{(t)}, \dots, C_{K}^{(t)}$ (and $\boldsymbol{\mu}_{1}^{(t)}, \dots, \boldsymbol{\mu}_{K}^{(t)}$).
                    \ENDIF
                \ENDFOR
                \ENDFOR
                \STATE \textbf{return} $C_{1}^{(T)}, \dots, C_{K}^{(T)}$ (and $\boldsymbol{\mu}_{1}^{(T)}, \dots, \boldsymbol{\mu}_{K}^{(T)}$).
            \end{algorithmic}
\end{algorithm}

Algorithm~\ref{ALGORITHM:LLOYD} (cf.~\cite[Algorithm~20.3]{MacK2003}) protocols one of the common formulations of the Lloyd's $K$-means algorithm -- with some trivial adjustments to make the indexing consistent with that of FW algorithm. Starting with some initial partition, the algorithm proceeds with forming $K$ clusters by assigning points $\boldsymbol{x}_{i}$ to the closest center (ties being broken arbitrarily). For all non-empty clusters, respective means are updated before proceeding to the next iteration. As initial partition, Algorithm~\ref{ALGORITHM:LLOYD} implements the usual Lloyd's initialization, but our convergence results to any other initialization strategy, e.g., based on optimized seeding strategies~\citep{arthur2007kmeans}. For any given run, the usual $\mathcal{O}(ndKT)$ complexity and $\mathcal{O}(nd + Kd)$ storage hold.

Our goal is to analyze Algorithm~\ref{ALGORITHM:LLOYD} through the prism of FW. To this end, every partition $\mathcal{C}$ of the dataset $\mathcal{D}$ is uniquely encoded by a vector $\boldsymbol{w} \in \mathcal{M}_{0}$ such that $w_{ik} = 1$ if and only if $\boldsymbol{x}_{i} \in C_{k}$, where
\begin{equation}
    \label{EQUATION:K_MEANS_BINARY_LATTICE}
    \mathcal{M}_{0} = \Big\{\boldsymbol{w} {\in} \{0, 1\}^{n {\times} K} \,\big|\,
    \sum_{k = 1}^{K} w_{ik} {=} 1, i {=} 1{,} \dots{,} n\Big\}.
\end{equation}
The SSE objective can be continuously extended to $\mathcal{M} = \operatorname{conv}(\mathcal{M}_{0})$ via
\begin{equation}
    f(\boldsymbol{w}) = \sum_{k = 1}^{K} \sum_{i = 1}^{n} w_{ik} \|\boldsymbol{x}_{i} - \bar{\boldsymbol{x}}_{k}^{\boldsymbol{w}}\|^{2}
\end{equation}
with $\boldsymbol{x}_{k}^{\boldsymbol{w}} := \frac{1}{n_{k}^{\boldsymbol{w}}} \sum\limits_{i = 1}^{n} w_{ik} \boldsymbol{x}_{i}$ and $n_{k}^{\boldsymbol{w}} := \sum\limits_{i = 1}^{n} w_{ik}$, where 
$\sum\limits_{i = 1}^{n} w_{ik} \|\boldsymbol{x}_{i} - \bar{\boldsymbol{x}}_{k}^{\boldsymbol{w}}\|^{2} := 0$
is set for each $k$ with $n_{k}^{\boldsymbol{w}} = 0$. Lemma~\ref{LEMMA:PROPERTY_OF_K_MEANS_OBJECTIVE} gives the gradient and the Hessian of $f$ (proof in Supplemental Section~\ref{SECTION:SUPPLEMENTAL_PROOFS}).

\begin{lemma}
    \label{LEMMA:PROPERTY_OF_K_MEANS_OBJECTIVE}
    For any $\boldsymbol{w} \in \mathcal{M}$ with $n_{k}^{\boldsymbol{w}} > 0$ for all $k$,
    $\nabla f$ and $\nabla^{2} f$ read as
    \begin{align*}
        \partial_{w_{ik}} f(\boldsymbol{w}) &= \|\boldsymbol{x}_{i} - \bar{\boldsymbol{x}}_{k}^{\boldsymbol{w}}\|^{2}
        \quad \text{ and } \\
        \partial_{w_{ik}} \partial_{w_{jl}} f(\boldsymbol{w}) &=
        -\frac{2}{n_{k}^{\boldsymbol{w}}} \cdot \mathds{1}_{\{k = l\}} \cdot (\boldsymbol{x}_{i} - \bar{\boldsymbol{x}}_{k}^{\boldsymbol{w}}) \cdot
        (\boldsymbol{x}_{j} - \bar{\boldsymbol{x}}_{k}^{\boldsymbol{w}}).
    \end{align*}
    Moreover, the Hessian $\nabla^{2} f$ is negative semidefinite.
\end{lemma}

Hence, $f$ is concave in the interior of $\mathcal{M}$ and, due to Bauer's minimum principle~\citep{kruzik2000bauer}, attains a global minimum at one of the extreme points $\mathcal{M}_{0}$. This also applies to any local minimum that either must be an extreme point itself or be ``tied'' with one. Unfortunately, $\nabla f$ can not be continuously extended to $\mathcal{M}$ so that the smooth FW algorithm cannot be applied to $f$ over $\mathcal{M}$.

\paragraph{Non-Empty Clusters.}

For illustration purposes, before analyzing the general case, we first consider the ``admissible'' version of $\mathcal{M}_{0}$ where all clusters are assumed non-empty, namely:
\begin{align}
        \notag
        \mathcal{M}_{0}^{\mathrm{adm}} = \Big\{\boldsymbol{w} \in &\{0, 1\}^{n \times K} \,\big|\,
        \sum_{k = 1}^{K} w_{ik} = 1, i = 1, \dots, n, \\
        \label{EQUATION:K_MEANS_BINARY_LATTICE_ADMISSIBLE}
        &\sum_{i = 1}^{n} w_{ik} \geq 1, k = 1, \dots, K\Big\},
\end{align}
and define the polytope $\mathcal{M}^{\mathrm{adm}} := \operatorname{conv}\big(\mathcal{M}_{0}^{\mathrm{adm}}\big)$. By Lemma~\ref{LEMMA:PROPERTY_OF_K_MEANS_OBJECTIVE}, $f$ is smooth in $\mathcal{M}^{\mathrm{adm}} \subset \mathbb{R}^{n \times K}$.

\begin{remark}
    \label{REMARK:CURVATURE_CONSTANT}
    Though irrelevant due to concavity, $\nabla^{2} f(\boldsymbol{w})$ is bounded for $\boldsymbol{w} \in \mathcal{M}^{\mathrm{adm}}$ (see Supplement):
    \begin{equation*}
        \|\nabla^{2} f(\boldsymbol{w})\| \leq 2R^{2}
        \quad \text{ with } \quad
        R = \max_{i, j = 1, \dots, n} \|\boldsymbol{x}_{i} - \boldsymbol{x}_{j}\|.
    \end{equation*}
    Therefore, $f$ has a bounded curvature constant
    \begin{equation*}
        C_{f} = 4 n \big(\max_{i, j = 1, \dots, n} \|\boldsymbol{x}_{i} - \boldsymbol{x}_{j}\|^{2}\big)
    \end{equation*}
    over $\mathcal{M}^{\mathrm{adm}}$.
    Thus, the general non-convex approach of~\cite{LaJu2016} is applicable but furnishes an inferior $\mathcal{O}(1/\sqrt{t})$ rate compared to the concave rate $\mathcal{O}(1/t)$ of~\cite{yurtsever2022cccp}.
\end{remark}

We endow $\mathbb{R}^{n \times K}$ with the usual Frobenius scalar product $\langle \boldsymbol{w}, \tilde{\boldsymbol{w}}\rangle_{\mathcal{F}} = \mathop{\operatorname{tr}}(\boldsymbol{w}' \tilde{\boldsymbol{w}})$. With the space $\mathcal{F}$ being isometrically isomorphic to $\mathbb{R}^{nK}$, \cite[Lemma~2.1]{yurtsever2022cccp} is applicable. Interestingly, since FW is affine-invariant~\citep{jaggi2013revisiting}, the choice of the scalar product does not affect the results.

The FW update in Algorithm~\ref{ALGORITHM:FW_UNIT_SS} with the unit step-size $\gamma_{t} = 1$ reads as
\begin{equation}
    \label{EQUATION:FRANK_WOLFE_UPDATE_K_MEANS_PROBLEM}
    \boldsymbol{w}^{(t + 1)} = \boldsymbol{s}^{(t)} \coloneqq
    \mathop{\operatorname{arg\,min}}_{\boldsymbol{s} \in \mathcal{M}^{\mathrm{adm}}} \langle \boldsymbol{s}, \nabla f(\boldsymbol{w}^{(t)})\rangle_{\mathcal{F}}.
\end{equation}
Using $\nabla f$ from Lemma~\ref{LEMMA:PROPERTY_OF_K_MEANS_OBJECTIVE}, the latter linear optimization problem  can be explicitly expressed as
\begin{equation}
    \label{EQUATION:FRANK_WOLFE_UPDATE_K_MEANS_PROBLEM_LINEAR_OPTIMIZATION}
    \min_{\boldsymbol{s} \in \mathcal{M}_{0}^{\mathrm{adm}}} \sum_{k = 1}^{K} \sum_{i = 1}^{n} s_{ik} \|\boldsymbol{x}_{i} - \bar{\boldsymbol{x}}_{k}^{\boldsymbol{w}^{(t)}}\|^{2}.
\end{equation}

The greedy solution $\tilde{\boldsymbol{s}}^{(t)} \in \mathcal{M}_{0}$ is given by
\begin{equation}
    \label{EQUATION:FRANK_WOLFE_UPDATE_K_MEANS_PROBLEM_LLOYDS_ALGORITHM_ASSIGNMENT}
    \begin{split}
        \tilde{s}^{(t)}_{ik} &=
        \begin{cases}
            1, & \bar{\boldsymbol{x}}_{k} = \mathop{\operatorname{arg\,min}}\limits_{k' = 1, \dots, K} \|\boldsymbol{x}_{i} - \bar{\boldsymbol{x}}_{k'}^{\boldsymbol{w}^{(t)}}\|^{2} \\
            0, &\text{otherwise}
        \end{cases} \\
        &=
        \begin{cases}
            1, & \bar{\boldsymbol{x}}_{k} = \mathop{\operatorname{arg\,min}}\limits_{k' = 1, \dots, K} \|\boldsymbol{x}_{i} - \boldsymbol{\mu}_{k'}^{(t)}\|^{2} \\
            0, &\text{otherwise}
        \end{cases}
    \end{split}
\end{equation}
where we observed
\begin{equation*}
    \boldsymbol{\mu}_{k}^{(t)} = \bar{\boldsymbol{x}}_{k}^{\boldsymbol{w}^{(t)}}
    \quad \text{ for } \quad \boldsymbol{w}^{(t)} \in \mathcal{M}^{\mathrm{adm}}.
\end{equation*}
Thus, the assignment in Equation~\eqref{EQUATION:FRANK_WOLFE_UPDATE_K_MEANS_PROBLEM_LLOYDS_ALGORITHM_ASSIGNMENT} equivalent with line~\ref{LINE:LLOYD_GREEDY_ASSIGNMENT} of Algorithm~\ref{ALGORITHM:LLOYD}.

For all $\boldsymbol{s} \in \mathcal{M}$, we can trivially estimate
\begin{align*}
    &\sum_{i = 1}^{n} \sum_{k = 1}^{K} s_{ik} \|\boldsymbol{x}_{i} - \boldsymbol{\mu}_{k}^{(t)}\|^{2} \geq \sum_{i = 1}^{n} \sum_{k = 1}^{K} s_{ik} \min_{k'} \|\boldsymbol{x}_{i} - \boldsymbol{\mu}_{k}^{(t)}\|^{2} \\
    &= \sum_{i = 1}^{n} \min_{k'} \|\boldsymbol{x}_{i} - \boldsymbol{\mu}_{k}^{(t)}\|^{2}
    = \sum_{i = 1}^{n} \sum_{k = 1}^{K} \tilde{s}_{ik}^{(t)} \|\boldsymbol{x}_{i} - \boldsymbol{\mu}_{k}^{(t)}\|^{2}.
\end{align*}
Thus, $\tilde{\boldsymbol{s}}^{(t)}$ solves Equation~\eqref{EQUATION:FRANK_WOLFE_UPDATE_K_MEANS_PROBLEM_LINEAR_OPTIMIZATION} over the superset $\mathcal{M}$. However, as long as $\tilde{\boldsymbol{s}}^{(t)} \in \mathcal{M}^{\mathrm{adm}}$, i.e., no cluster is empty, it also solves the original problem~\eqref{EQUATION:FRANK_WOLFE_UPDATE_K_MEANS_PROBLEM_LINEAR_OPTIMIZATION}. 
Invoking~\cite[Lemma~2.1]{yurtsever2022cccp}, the FW gap $g_{t}$ is equal to the decrease in the objective $f$ at step $0 \leq t \leq T - 1$:
\begin{equation*}
    \begin{split}
        &g_{t} = 
        \big\langle \boldsymbol{d}_{t}, - \nabla f(\boldsymbol{w}^{(t)})\big\rangle =
        \big\langle \boldsymbol{w}^{(t+1)} - \boldsymbol{w}^{(t)}, - \nabla f(\boldsymbol{x}^{(t)})\big\rangle \\
        &= 
        \sum_{k = 1}^{K} \sum_{i = 1}^{n} w_{ik}^{(t)} \|\boldsymbol{x}_{i} - \bar{\boldsymbol{x}}_{k}^{\boldsymbol{w}^{(t)}}\|^{2} \\
        &-
        \sum_{k = 1}^{K} \sum_{i = 1}^{n} w_{ik}^{(t + 1)} \|\boldsymbol{x}_{i} 
        - \bar{\boldsymbol{x}}_{k}^{\boldsymbol{w}^{(t)}}\|^{2}
        = f\big(\boldsymbol{x}^{(t)}\big) - f\big(\boldsymbol{x}^{(t + 1)}\big).
    \end{split}
\end{equation*}

With the preparations above, the next Theorem is a direct consequence of~\cite[Lemma~2.1]{yurtsever2022cccp}---or our more general Theorem~\ref{THEOREM:CONVERGENCE}---applied to the concave SSE objective of the $K$-means algorithm.

\begin{theorem}
    \label{THEOREM:K_MEAN_SMOOTH}

    Let $\boldsymbol{w}^{(t)}$, $0 \leq t \leq T$, be produced by Lloyd's $K$-means algorithm. Assuming $\boldsymbol{w}^{(t)} \in \mathcal{M}_{0}^{\mathrm{adm}}$ for all $t$, the membership vector $\boldsymbol{w}^{(t)}$ coincides with the the $t$-th iterate of the FW Algorithm~\ref{ALGORITHM:FW_UNIT_SS} applied to minimizing $f$ over $\mathcal{M}^{\mathrm{adm}}$ so that the running minimum FW gap satisfies
    \begin{equation*}
        \tilde{g}_{t} = \min_{0 \leq \tau \leq t} g_{\tau} =
        \min_{0 \leq \tau \leq t} \Delta \mathrm{SSE}_{\tau}
        \leq \frac{h_{0}}{t + 1}
    \end{equation*}
    for $t = 0, \dots, T - 1$, where $h_{0} = f(\boldsymbol{w}^{(0)}) - \displaystyle \min_{\boldsymbol{w} \in \mathcal{M}} f(\boldsymbol{w})$ and $\Delta \mathrm{SSE}_{t} = f\big(\boldsymbol{w}^{(t)}\big) - f\big(\boldsymbol{w}^{(t + 1)}\big)$.
\end{theorem}

It should be emphasized that the (unit) constant in Theorem~\ref{THEOREM:K_MEAN_SMOOTH} is independent of the sample size $n$, the dimension $p$ or initial clusters. Unlike proofs based on $\varepsilon$-coreset, no dependence on $\varepsilon$ occurs. The initial suboptimality $h_{0}$ itself scales linearly in $n$ and $p$. Invoking the standard SSE identity~\cite[p.~693]{JoWi2007} and $f(\boldsymbol{w}) \geq 0$, the initial gap $h_{0} := f(\boldsymbol{w}^{(0)}) - \min\limits_{\boldsymbol{w} \in \mathcal{M}} f(\boldsymbol{w})$ can easily be estimated via
\begin{equation}
    \label{EQUATION:INITIAL_GAP_ESTIMATE}
    h_{0} \leq (n - 1) \operatorname{tr}(\boldsymbol{S}_{n}) \equiv
    \sum_{i = 1}^{n} \|\boldsymbol{x}_{i} - \bar{\boldsymbol{x}}\|^{2}
\end{equation}
where $\bar{\boldsymbol{x}} = \frac{1}{n} \sum\limits_{i = 1}^{n} \boldsymbol{x}_{i}$ and
$\boldsymbol{S}_{n} := \frac{1}{n - 1} \sum\limits_{i = 1}^{n} (\boldsymbol{x}_{i} - \bar{\boldsymbol{x}}) (\boldsymbol{x}_{i} - \bar{\boldsymbol{x}})'$ 
are the usual sample mean vector and (unbiased) sample covariance matrix of the entire dataset $\mathcal{D}$.

\begin{remark} \label{REMARK:INITIAL_SUBOPTIMALITY_LLN_ESTIMATE}
    If the data $\boldsymbol{x}_{i}$'s are independently sampled from a squared integrable distribution on a probability space $(\Omega, \mathcal{F}, \mathbb{P})$ with covariance matrix $\boldsymbol{\Sigma}$, the strong law of large numbers implies
    \begin{equation*}
        h_{0} = (n - 1) \operatorname{tr}\big(\boldsymbol{\Sigma}\big) + o_{\mathbb{P}}(n) \quad \mathbb{P}\text{-a.s.~as } n \to \infty
    \end{equation*}
    uniformly over arbitrary initial cluster choice.
\end{remark}

\paragraph{General Case.}

If Lloyd's greedy assignment fails to produce admissible cluster membership, the linear programming problem in Equation~\eqref{EQUATION:FRANK_WOLFE_UPDATE_K_MEANS_PROBLEM} furnishes an alternative form of Lloyd's update that guarantees the cluster allocation stays ``admissible'' at any time. Since no closed-form closed solution exists, Thorup's version of the Hungarian algorithm~\citep{Tho2004} can be used to compute the solution in an $\mathcal{O}(n K^{2} + K^{2} \log\log(K))$-time. This approach is generally disfavored in most communities. In addition to computational considerations, the rationale is best illustrated by non-spherical Gaussian mixture models. Indeed, even clusters with $d + 1$ or more points in general position are not guaranteed to produce well-conditioned covariances~\citep{garcia2008general}. Thus, constraints on cluster sizes, including non-empty clusters, are often viewed as ineffective.

Instead of amending the Lloyd's $K$-means algorithm, we will pursue a different approach here by analyzing it through the framework of semismooth FW Algorithm~\ref{ALGORITHM:FW_UNIT_SS}. In addition to being continuous on $\mathcal{M}$, $f$ is concave and continuously differentiable for all $\boldsymbol{w} \in \mathcal{M}$ with $n_{k}^{\boldsymbol{w}} > 0$ for all $k = 1, \dots, K$ and, therefore, semismooth on $\mathcal{M}$.

For any $\boldsymbol{w} \in \mathcal{M}$ with $n_{k}^{\boldsymbol{w}} = 0$ for at least one (but at most $K - 1$) $k$, let $\boldsymbol{c}_{k} \in \mathcal{M}$ denote a prescribed ``ephemerous'' center of the $k$-th cluster. Consider any sequence $(\boldsymbol{w}_{j})_{j} \subset \operatorname{int}(\mathcal{M})$ such that
\begin{equation}
    \lim_{j \to \infty} \boldsymbol{w}_{j} = \boldsymbol{w} \quad \text{ and } \quad
    \lim_{j \to \infty} \bar{\boldsymbol{x}}_{k}^{\boldsymbol{w}_{j}} = \boldsymbol{c}_{k}
\end{equation}
for every $k$ with $n_{k}^{\boldsymbol{w}} = 0$.
Such selection is easily possible since we can write 
\begin{equation*}
    \boldsymbol{c}_{k} = \frac{1}{\sum\limits_{i = 1}^{n} \alpha_{ik}} \sum\limits_{i = 1}^{n} \alpha_{ik} \boldsymbol{x}_{i}    
\end{equation*}
for some $\alpha_{ik} \in [0, 1]$ and rescale $\alpha_{ik}$'s to enforce the condition $\sum\limits_{i = 1}^{n} \alpha_{ik} = o(1)$ while adjusting other columns by $o(1)$ to maintain the remaining constraints in $\mathcal{M}$. By Lemma~\ref{LEMMA:PROPERTY_OF_K_MEANS_OBJECTIVE},
\begin{equation*}
    \lim_{j \to \infty} \partial_{w_{ik}} f(\boldsymbol{w}^{j}) =
    \begin{cases}
        \|\boldsymbol{x}_{i} - \bar{\boldsymbol{x}}_{k}^{\boldsymbol{w}}\|^{2}, & n_{k}^{\boldsymbol{w}} > 0, \\
        \|\boldsymbol{x}_{i} - \boldsymbol{c}_{k}\|^{2}, &\text{otherwise}.
    \end{cases}
\end{equation*}
As a directional derivative, the latter is an element of $\partial f(\boldsymbol{w})$. Importantly, this is exactly the subgradient $\boldsymbol{r}^{(t)} \in \partial f(\boldsymbol{w}^{(t)})$ used by Lloyd's greedy LMO over $\mathcal{M}$ as part of Algorithm~\ref{ALGORITHM:LLOYD} if we let $\boldsymbol{c}_{k} = \boldsymbol{\mu}_{k}^{(t)}$ at the $t$-th step.
Moreover, we also have $\boldsymbol{\mu}_{k}^{(t)} = \bar{\boldsymbol{x}}_{k}^{\boldsymbol{w}^{(t)}}$ if $n_{k}^{\boldsymbol{w}^{(t)}} > 0$.

Thus, we arrive at a nonsmooth generalization of Theorem~\ref{THEOREM:K_MEAN_SMOOTH}, for which we can now drop the assumption of $\boldsymbol{w}^{(t)} \in \mathcal{M}^{\mathrm{adm}}$ for $t = 0, \dots, T$ while retaining the same convergence estimate.

\paragraph{Practical Implications.}

Our theoretical developments put forward a new {\it bona fide} error measure
\begin{equation}
    \label{EQUATION:OUR_GAP_ERROR_MEASURE}
    g_{t} \equiv \Delta \mathrm{SSE}_{t}
\end{equation}
that can be used in stopping decisions when running the $K$-means algorithm. Presently, the root of the sum of squared distances between new and old centroids
\begin{equation}
    \label{EQUATION:COMMON_ERROR_MEASURE}
    \epsilon_{t} :=
    \Big(\sum_{k = 1}^{K} \big\|\boldsymbol{\mu}_{k}^{(t + 1)} - \boldsymbol{\mu}_{k}^{(t)}\big\|^{2}\Big)^{1/2}
\end{equation}
is typically employed as an error measure. 

Unlike our FW gap-based measure~\citep{LaJu2016}, the error measure in Equation~\eqref{EQUATION:COMMON_ERROR_MEASURE} is not affine invariant and has no known convergence rate. Further, the choice of the tolerance in the stopping rule is problematic since the right scale is difficult to guess, which may lead to premature stopping or longer than necessary runs. In contrast, two natural choices of the tolerance threshold for our FW gap measure are
\begin{equation}
    \label{EQUATION:TOLERANCE_CHOICE}
    \mathrm{tol} = \varepsilon \mathrm{SSE}_{0} \text{ or } \mathrm{tol} = \varepsilon (n - 1) \mathop{\operatorname{tr}}(\boldsymbol{S}_{n})
\end{equation}
(cf.~Equation~\eqref{EQUATION:INITIAL_GAP_ESTIMATE})
for some small $\varepsilon > 0$, say, $\varepsilon = 10^{-6}$,
where the sample covariance $\boldsymbol{S}_{n}$ can readily be obtained from the data. Last but not least, on the strength of Theorem~\ref{THEOREM:K_MEAN_SMOOTH}, the $K$-means algorithm is further guaranteed to terminate in no more than
$\lceil \frac{\mathrm{SSE}_{0}}{\mathrm{tol}} \rceil$ steps. No such estimates were previously available.

\section{\MakeUppercase{Simulation Study}}
\label{SECTION:SIMULATION_STUDY}

We want to illustrate the results of Section~\ref{SECTION:K_MEANS_CLUSTERING} with a simulation study. The major numerical challenge in studying the convergence rate of Algorithm~\ref{ALGORITHM:LLOYD} is the fact that the iteration becomes stationary. Since our results are \emph{non-asymptotic}, the strategy is to study the numerical convergence in the preasymptotic regime. For a given dataset and initial seeds, the latter regime can be too short to facilitate any reasonable statistical analysis. Therefore, we rather focus on the \emph{worst-case} performance over randomly selected seeds -- for a fixed dataset or over multiple datasets sampled from a given population. Letting $g_{t}^{\mathrm{WC}}$ denote the worst-case (i.e., largest) FW gap (over all replications) at iteration $t$, from Figures~\ref{FIG:FW_CONVERGENCE_RATE} and~\ref{FIG:SEGMENTATION_DATA_CONVERGENCE_RATE}, we observe that it tends to be monotonically decreasing in the preasymptotic regime so that we chose to study the later in lieu of the running minimum asymptotic gap $\tilde{g}_{t}$ appearing in Theorem~\ref{THEOREM:CONVERGENCE}. Empirically, the convergence rate of $g_{t}^{\mathrm{WC}}$ was estimated by applying OLS to the linear regression model
\begin{equation*}
    \log(g_{t}^{\mathrm{WC}}) = \beta_{0} + \beta_{1} \log(t + 1) + \varepsilon
\end{equation*}
assuming i.i.d.~normal errors $\varepsilon$.
Preasymptotic regime was defined as the lower two thirds of the $\log(t)$ range. Theoretical values were assumed $\beta_{1} = -1$ and $\beta_{0} = (n - 1) \operatorname{tr}(\boldsymbol{\Sigma})$, where the worst-case total variance was estimated empirically. All codes were implemented in plain Python and run in CPU mode on a 64-bit Ubuntu system on a Dell Precision 7960 Tower (Intel\textregistered{} Xeon(R) w7-3465X $\times$ 56 and 128 GB RAM).

\paragraph{Synthetic data.}

\begin{figure}
    \centering
    \begin{minipage}[t]{0.5\textwidth}
        \includegraphics[width=1.0\linewidth]{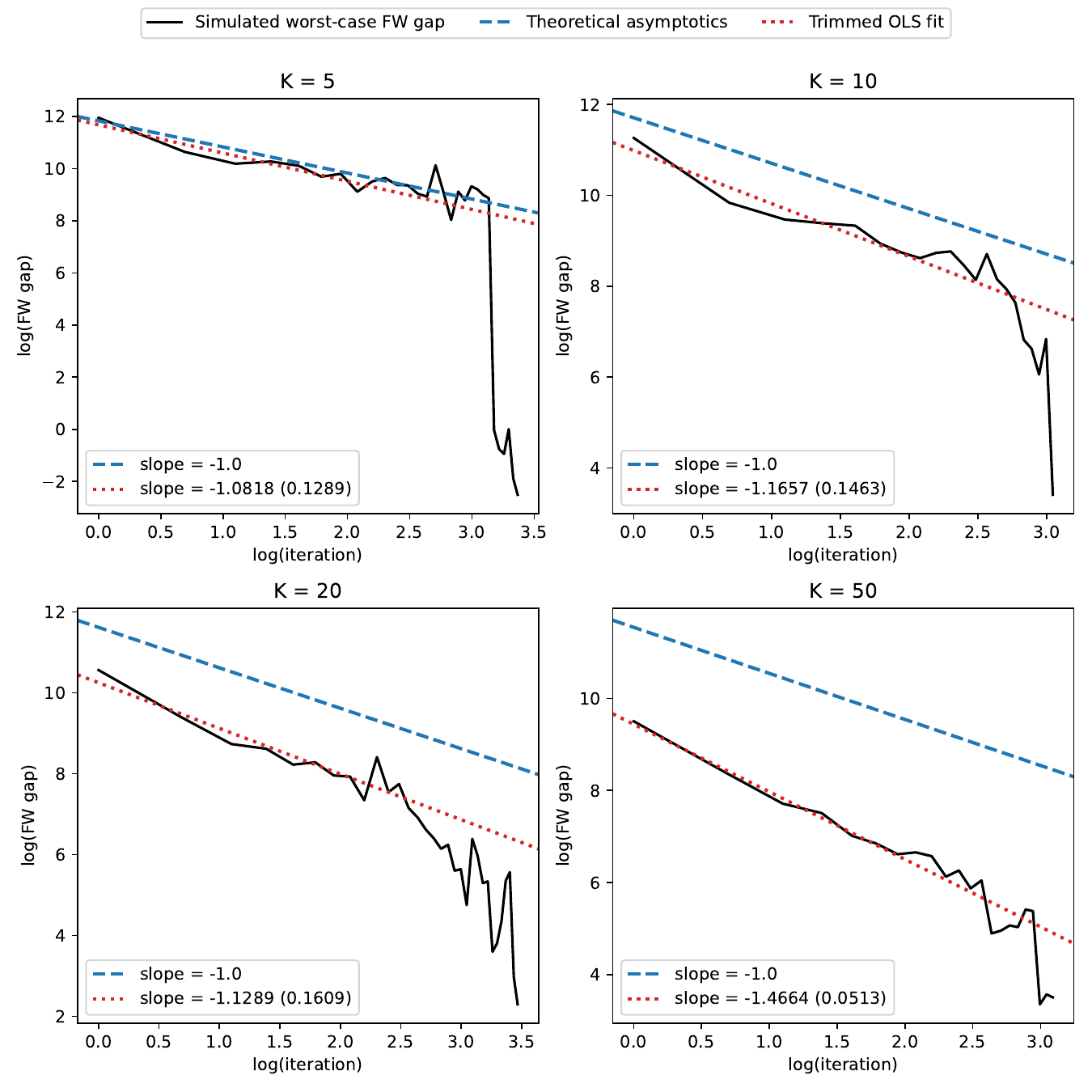}
    \end{minipage}
    \caption{Worst-case convergence rate on synthetic ``blobs'' data for $n = 500$ and $d = 5$.}
    \label{FIG:FW_CONVERGENCE_RATE}
\end{figure}

Consider the Gaussian mixture model $\sum\limits_{k = 1}^{K} \pi_{k} \varphi_{k}(\boldsymbol{x}|\boldsymbol{\mu}_{k}, \sigma_{k}^{2} \boldsymbol{I})$ with equal proportions $\pi_{k} = \frac{1}{K}$ and equal spherical covariances with $\sigma_{k} = 1$ as implemented in~\texttt{make\_blobs()} of Python's \texttt{scikit-learn}. Whereas this is the usual assumption behind the $K$-means algorithm~\citep[Chapter~9]{bishop2006pattern}, real-world datasets typically depart from this model, i.e., the data could come from an elliptic mixture, not a spherical one. Aside from statistical implications, model violations tend to increase the runtime. Some authors even adopt uniformly distributed data~\citep{hamerly2010making} as a hypothetical worst-case scenario of ``non-existing'' cluster structure. Although we consider spherical mixtures, the clusters are rather poorly separated. Therefore, for small $n$, the data empirically look as if they were uniformly distributed.

For each $(n, d, K)$ pair with $n = 500, 1000, 5000$, $d = 2, 5, 10$ and $K = 5, 10, 20, 50$, we performed $10{,}000$ Monte Carlo (MC) simulations to estimate $g_{t}^{\mathrm{WC}}$ by randomly sampling both the data (from the `blobs' mixture model) and initial seeds (from the dataset). The total runtime added up to about 11 hrs. Figure~\ref{FIG:FW_CONVERGENCE_RATE} illustrates one of the plots obtained, namely for $n = 500$ and $d = 5$. The discrepancy between the two lines appears to be smaller for lower values of $K$ and $d$ and larger values of $n$, suggesting improved ability to find better local minima in the latter case.

Figure~\ref{FIG:FW_CONVERGENCE_RATE}
displays the worst-case (i.e., maximum over all replications) FW gap plot vs $t$ in a log-log plot.
Applying an upper-tailed Student's $t$-test of $H_{0}: \beta_{1} = -1$ vs \hbox{$H_{1}: \beta_{1} > -1$} across all scenarios considered, the $p$-values never dropped below $1 - (1 - 0.05)^{1/36} \approx 0.0014$, which is the \v{S}id\'{a}k-adjusted limit to maintain a family-wise $0.05$ test size, the composite null hypothesis failed to be rejected, corroborating that the worst-case preasymptotic rate was never slower than $\mathcal{O}(1/t)$. The estimated sloped are reported along with standard errors (in parentheses).
The full set of plots as well as estimated slopes, standard errors and $p$-values are provided in Supplemental Sections~\ref{APPENDIX:ADDITIONAL_FIGURES} and~\ref{APPENDIX:TABLES}.

\paragraph{Image Segmentation Dataset.}

Consider the
image segmentation dataset 
available from UCI Machine Learning Repository~\citep{Dua:2025} recently studied by~\cite{PoJo2022} in the context of robust estimation. 
Pooling training and test data, each of $2{,}310$ instances is classified as one of the seven classes: \texttt{BRICKFACE}, \texttt{CEMENT}, \texttt{FOLIAGE}, \texttt{GRASS}, \texttt{PATH}, \texttt{SKY} and \texttt{WINDOW}. 
Keeping only continuous features (columns 6 through 19), we end up with 14 variables ($d = 14$). \cite{PoJo2022} empirically demonstrated that the 7-component mixture is non-Gaussian. Specifically, by analyzing the empirical quantiles of robust squared Mahalanobis distances of the \texttt{SKY} component, they showed that the former significantly exceeded respective $\chi^{2}_{p}$-quantiles expected for Gaussian data. In addition to non-Gaussianity, the assumption of equal (not to mention spherical) covariances is also violated in this dataset. In the spirit of these observations, the image segmentation dataset is principally different from synthetic ``blobs data'' used in the previous simulation.

\begin{figure}
    \centering
    \begin{minipage}[t]{0.45\textwidth}
        \includegraphics[width=1.0\linewidth]{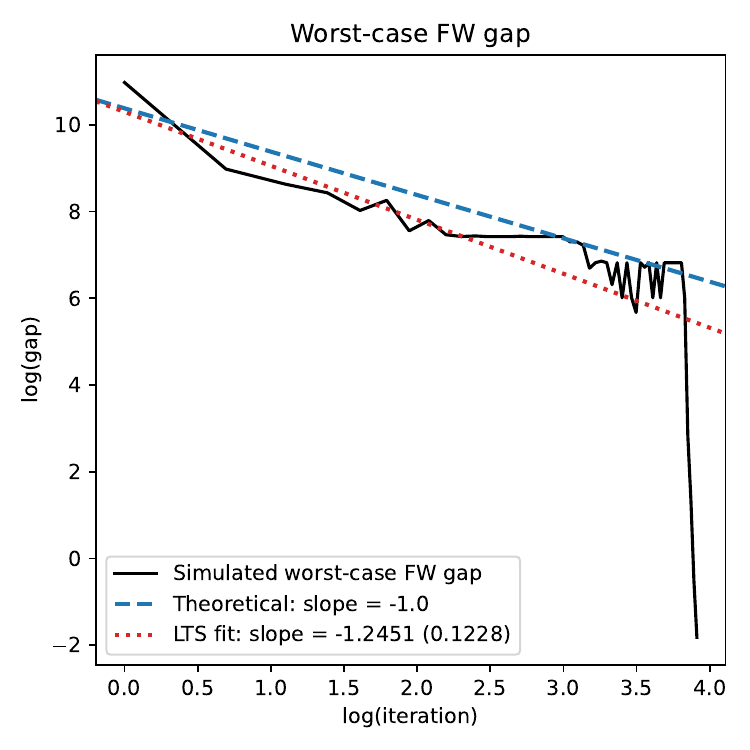}
    \end{minipage}
    \caption{Worst-case convergence rate on image segmentation data.}
    \label{FIG:SEGMENTATION_DATA_CONVERGENCE_RATE}
\end{figure}

In a similar fashion, an MC simulation over $50{,}000$ random seeds was performed for $K = 7$ with a total runtime of about 1 hr. 
The results are displayed in Figure~\ref{FIG:SEGMENTATION_DATA_CONVERGENCE_RATE}. The slope estimate is reported along with the standard error (in parentheses). Based on an upper-tailed Student's $t$-test with $\hat{\beta}_{1} = -1.2451$ and $\mathrm{se}(\hat{\beta}_{1}) = 0.1228$, the null hypothesis $H_{0}:\beta_{1} = -1$ failed to be rejected against $H_{1}: \beta_{1} > -1$ at the test size of $\alpha = 0.05$, empirically corroborating an $\mathcal{O}(1/t)$ worst-case rate of Lloyd's $K$-means on this dataset.

\paragraph{Other Datasets.}
Other popular datasets (e.g., Fisher's Iris data, Swiss banknote data, etc.) as well as some datasets from~\cite{hamerly2010making} were analyzed in a fashion similar to the image segmentation dataset. The results were consistent with the ones reported above.

\paragraph{Computer Codes.}
The \texttt{Python} codes used to produce the empirical results reported in this work are available at~\url{https://github.com/mpokojovy/kmeansFW}. 

\section{\MakeUppercase{Discussion \& Implications}}

\paragraph{Discussion.} We reiterate and emphasize the importance of non-asymptotic convergence guarantees for Lloyd's $K$-means -- or any other machine learning algorithm for that matter -- that hold in a uniform fashion with respect to the sample size, space dimension, hyperparameters, initial seeding or optimization ``intricacies'' like the curvature constant or geometric aspects of the domain. Building upon the recent precedent~\citep{yurtsever2022cccp} in the context of convex-concave procedure (CCCP), we proved that Lloyd's algorithm is a special case of FW. In doing so, we not only were able to perfectly carry over their results to Lloyd's $K$-means clustering but developed an FW variant for general semismooth concave objectives and proved an $\mathcal{O}(1/t)$ convergence rate. Owing to concavity, we could employ the less technical Clarke's instead of Goldstein's subdifferential prevalent in the nonsmooth literature allowing for better transparency and broader accessibility. 

\paragraph{Practical Implications.} By characterizing the drop in SSE at each step of the $K$-means algorithm as an FW gap and, thus, a type of Bregman divergence, we establish a new affine-invariant~\citep{LaJu2016} alternative to existing error measures. Our measure allows for a simple stopping rule based on a natural choice of the tolerance threshold $\mathrm{tol}$ (cf.~Equation~\eqref{EQUATION:TOLERANCE_CHOICE}). Further, the maximum number of steps can be easily estimated. These discoveries can prove helpful both in traditional and emerging application domains, e.g., in federated multi-view clustering~\citep{liu2023communication}.

\paragraph{Extensions.}

With minor modifications, our results are applicable to trimmed $K$-means~\citep{dorabiala2022robust, garcia2008general}. As for the semismooth FW algorithm with concave objectives, an extension to infinite-dimensional Hilbert spaces is possible paving the way for optimal steering of dynamical systems via bang-bang control~\citep{seyde2021bangbang}.

\paragraph{Limitations.}

The limitations of our study are two-fold. First, while our approach is applicable to a wide range of $K$-means variants and robust extensions, online and stochastic configurations are less obvious. The same applies to general elliptic mixtures that would require a better understanding of the negative trimmed log-likelihood-type objective~\citep{garcia2008general}. Second, despite reassuring results, larger-scale simulations would be desirable in future investigations.

\section*{\MakeUppercase{Acknowledgements}}
This work was partially supported by the National Science Foundation (DMS-2402544) and by the Canada CIFAR AI Chair Program. Simon Lacoste-Julien is a CIFAR Associate Fellow in the Learning in Machines \& Brains Program.

\bibliographystyle{chicago}
\bibliography{bibliography}
\typeout{}

% \begin{thebibliography}{}
% \setlength{\itemindent}{-\leftmargin}
% \makeatletter\renewcommand{\@biblabel}[1]{}\makeatother
% \bibitem{} J.~Alspector, B.~Gupta, and R.~B.~Allen (1989).
%     \newblock Performance of a stochastic learning microchip.
%     \newblock In D. S. Touretzky (ed.),
%     \textit{Advances in Neural Information Processing Systems 1}, 748--760.
%     San Mateo, Calif.: Morgan Kaufmann.

% \bibitem{} F.~Rosenblatt (1962).
%     \newblock \textit{Principles of Neurodynamics.}
%     \newblock Washington, D.C.: Spartan Books.

% \bibitem{} G.~Tesauro (1989).
%     \newblock Neurogammon wins computer Olympiad.
%     \newblock \textit{Neural Computation} \textbf{1}(3):321--323.
% \end{thebibliography}

%%%%%%%%%%%%%%%%%%%%%%%%%%%%%%%%%%%%%%%%%%%%%%%%%%%%%%%%%%%%

\clearpage
\newpage

\section*{Checklist}

% % %%% BEGIN INSTRUCTIONS %%%
% The checklist follows the references. For each question, choose your answer from the three possible options: Yes, No, Not Applicable.  You are encouraged to include a justification to your answer, either by referencing the appropriate section of your paper or providing a brief inline description (1-2 sentences). 
% Please do not modify the questions.  Note that the Checklist section does not count towards the page limit. Not including the checklist in the first submission won't result in desk rejection, although in such case we will ask you to upload it during the author response period and include it in camera ready (if accepted).

% \textbf{In your paper, please delete this instructions block and only keep the Checklist section heading above along with the questions/answers below.}
% % %%% END INSTRUCTIONS %%%

\begin{enumerate}

  \item For all models and algorithms presented, check if you include:
  \begin{enumerate}
    \item A clear description of the mathematical setting, assumptions, algorithm, and/or model. [Yes]

    {\it Justification:} The paper is mostly self-contained and provides a clear and thorough description of all mathematical settings, assumptions and algorithms. When necessary, specific references are provided for extra details.
    
    \item An analysis of the properties and complexity (time, space, sample size) of any algorithm. [Yes]

    {\it Justification:} While FW is an abstract algorithm so that the complexity can only be measured with respect to the number of steps $T$, the usual complexity and storage hold for the $K$-means algorithm (see the paragraph below Algorithm~\ref{ALGORITHM:LLOYD} in Section~\ref{SECTION:K_MEANS_CLUSTERING}).
    
    \item (Optional) Anonymized source code, with specification of all dependencies, including external libraries. [Yes]

    {\it Justification:} Anonymized run-ready source code that solely relies on standard libraries is provided in the Supplement.
  \end{enumerate}

  \item For any theoretical claim, check if you include:
  \begin{enumerate}
    \item Statements of the full set of assumptions of all theoretical results. [Yes]

    {\it Justification:} All lemmas and theorems contain the full set of assumptions.
    
    \item Complete proofs of all theoretical results. [Yes]

    {\it Justification:} Complete proofs of all theoretical results are provided in the paper and the Supplement.
    
    \item Clear explanations of any assumptions. [Yes]

    {\it Justification:} All assumptions are clearly explained. Additional remarks are often provided to facilitate better understanding.
  \end{enumerate}

  \item For all figures and tables that present empirical results, check if you include:
  \begin{enumerate}
    \item The code, data, and instructions needed to reproduce the main experimental results (either in the supplemental material or as a URL). [Yes]

    {\it Justification:} All code, data and instructions (including seeds for random number generation) are included in the Supplement.
    
    \item All the training details (e.g., data splits, hyperparameters, how they were chosen). [Yes]

    {\it Justification:} This and other information is included in the supplemental code.
    
    \item A clear definition of the specific measure or statistics and error bars (e.g., with respect to the random seed after running experiments multiple times). [Yes]

    {\it Justification:} Standard errors and $p$-values are reported and explained in Section~\ref{SECTION:SIMULATION_STUDY} of the paper and the Supplement.
    
    \item A description of the computing infrastructure used. (e.g., type of GPUs, internal cluster, or cloud provider). [Yes]

    {\it Justification:} This information is provided in Section~\ref{SECTION:SIMULATION_STUDY} of the paper.
  \end{enumerate}

  \item If you are using existing assets (e.g., code, data, models) or curating/releasing new assets, check if you include:
  \begin{enumerate}
    \item Citations of the creator If your work uses existing assets. [Yes]

    {\it Justification:} See~\cite{Dua:2025}.
    
    \item The license information of the assets, if applicable. [Yes]

    {\it Justification:} Following the link provided (see~\cite{Dua:2025}), the data are licensed under the Creative Commons Attribution 4.0 International (CC BY 4.0) license.
    
    \item New assets either in the supplemental material or as a URL, if applicable. [Not Applicable]

    {\it Justification:} No new assets introduced.
    
    \item Information about consent from data providers/curators. [Not Applicable]

    {\it Justification:} Not required under CC BY 4.0.
    
    \item Discussion of sensible content if applicable, e.g., personally identifiable information or offensive content. [Not Applicable]

    {\it Justification:} The dataset is publicly available and does not include any sensible content.
  \end{enumerate}

  \item If you used crowdsourcing or conducted research with human subjects, check if you include:
  \begin{enumerate}
    \item The full text of instructions given to participants and screenshots. [Not Applicable]

    {\it Justification:} No participants were involved.
    
    \item Descriptions of potential participant risks, with links to Institutional Review Board (IRB) approvals if applicable. [Not Applicable]

    {\it Justification:} No participants were involved.
    
    \item The estimated hourly wage paid to participants and the total amount spent on participant compensation. [Not Applicable]

    {\it Justification:} No participants were involved.
  \end{enumerate}

\end{enumerate}

\clearpage
\appendix
\thispagestyle{empty}

% Supplementary material: To improve readability, you must use a single-column format for the supplementary material.
\onecolumn
\aistatstitle{%Instructions for Paper Submissions to AISTATS 2026: \\
Supplementary Materials}

\section{\MakeUppercase{Supplementary Proofs}}

\label{SECTION:SUPPLEMENTAL_PROOFS}

\subsection{Proof of Lemma~\ref{LEMMA:PROPERTY_OF_K_MEANS_OBJECTIVE}}

\begin{proof}
    Observing
	\begin{align*}
        \partial_{w_{ik}} n_{k}^{\boldsymbol{w}} &= 1, \quad
        \partial_{w_{ik}} \bar{\boldsymbol{x}}_{k}^{\boldsymbol{w}} =
        -\frac{1}{(n_{k}^{\boldsymbol{w}})^{2}} \sum_{i' = 1}^{n} w_{i'k} \boldsymbol{x}_{i'} 
        + \frac{1}{n_{k}^{\boldsymbol{w}}} \boldsymbol{x}_{i}
        =
        \frac{1}{n_{k}^{\boldsymbol{w}}} \big(\boldsymbol{x}_{i} - \bar{\boldsymbol{x}}_{k}^{\boldsymbol{w}}\big), \\
        \partial_{w_{ik}} \partial_{w_{jk}} \bar{\boldsymbol{x}}_{k}^{\boldsymbol{w}} &=
        -\frac{1}{(n_{k}^{\boldsymbol{w}})^{2}} \big(\boldsymbol{x}_{i} - \bar{\boldsymbol{x}}_{k}^{\boldsymbol{w}}\big) 
        -\frac{1}{n_{k}^{\boldsymbol{w}}} \partial_{w_{jk}} \bar{\boldsymbol{x}}_{k}^{\boldsymbol{w}}
        = -\frac{1}{n_{k}^{\boldsymbol{w}}} \Big(\big(\boldsymbol{x}_{i} - \bar{\boldsymbol{x}}_{k}^{\boldsymbol{w}}\big)
        + \big(\boldsymbol{x}_{j} - \bar{\boldsymbol{x}}_{k}^{\boldsymbol{w}}\big)\Big),
	\end{align*}
	we compute the first-order
	\begin{align*}
        \partial_{w_{ik}} f(\boldsymbol{w}) &=
        \partial_{w_{ik}} \Big(\sum_{k' = 1}^{K} \sum_{i' = 1}^{n} 
        w_{i'k'} \|\boldsymbol{x}_{i'} - \bar{\boldsymbol{x}}_{k'}^{\boldsymbol{w}}\|^{2}\Big) \\
        &= \|\boldsymbol{x}_{i} - \bar{\boldsymbol{x}}_{k}^{\boldsymbol{w}}\|^{2} -
        \frac{2}{n_{k}^{\boldsymbol{w}}} (\boldsymbol{x}_{i} - \bar{\boldsymbol{x}}_{k}^{\boldsymbol{w}}) \cdot
        \Big(\sum_{i' = 1}^{n} w_{i'k} (\boldsymbol{x}_{i'} - \bar{\boldsymbol{x}}_{k}^{\boldsymbol{w}})\Big)
        = \|\boldsymbol{x}_{i} - \bar{\boldsymbol{x}}_{k}^{\boldsymbol{w}}\|^{2}
	\end{align*}
	and the second-order derivatives
	\begin{align*}
        \partial_{w_{ik}} \partial_{w_{jl}} f(\boldsymbol{w})
        &= -2 (\boldsymbol{x}_{i} - \bar{\boldsymbol{x}}_{k}^{\boldsymbol{w}}) \cdot \partial_{w_{jl}} \bar{\boldsymbol{x}}_{k}^{\boldsymbol{w}}
        = -\frac{2}{n_{k}^{\boldsymbol{w}}} \cdot \mathds{1}_{\{k = l\}} \cdot (\boldsymbol{x}_{i} - \bar{\boldsymbol{x}}_{k}^{\boldsymbol{w}}) \cdot
        (\boldsymbol{x}_{j} - \bar{\boldsymbol{x}}_{k}^{\boldsymbol{w}}).
	\end{align*}
	For $\boldsymbol{\xi} \in \mathbb{R}^{n \times K}$, we can write
	\begin{align*}
        \big\langle \nabla^{2} f(\boldsymbol{w}) \boldsymbol{\xi}, \boldsymbol{\xi}\big\rangle_{\mathcal{F}} 
        &= \sum_{k = 1}^{K} \sum_{i = 1}^{n} \xi_{ik} \sum_{i' = 1}^{n} \Big(
        -\frac{2}{n_{k}^{\boldsymbol{w}}} (\boldsymbol{x}_{i} - \bar{\boldsymbol{x}}_{k}^{\boldsymbol{w}}) \cdot
        (\boldsymbol{x}_{i'} - \bar{\boldsymbol{x}}_{k}^{\boldsymbol{w}})\Big) \xi_{i'k} \\
        &=
        -2 \sum_{k = 1}^{K} \frac{1}{n_{k}^{\boldsymbol{w}}} \sum_{i = 1}^{n} \sum_{i' = 1}^{n}
        \xi_{ik} (\boldsymbol{x}_{i} - \bar{\boldsymbol{x}}_{k}^{\boldsymbol{w}}) \cdot
        (\boldsymbol{x}_{i'} - \bar{\boldsymbol{x}}_{k}^{\boldsymbol{w}}) \xi_{i'k} \\
        &=
        -2 \sum_{k = 1}^{K} \frac{1}{n_{k}^{\boldsymbol{w}}} 
        \Big(\sum_{i = 1}^{n} \xi_{ik} (\boldsymbol{x}_{i} - \bar{\boldsymbol{x}}_{k}^{\boldsymbol{w}})\Big) \cdot
        \Big(\sum_{i' = 1}^{n} \xi_{ik} (\boldsymbol{x}_{i'} - \bar{\boldsymbol{x}}_{k}^{\boldsymbol{w}})\Big) \\
        &=
        -2 \sum_{k = 1}^{K} \frac{1}{n_{k}^{\boldsymbol{w}}} 
        \Big\|\sum_{i = 1}^{n} \xi_{ik} (\boldsymbol{x}_{i} - \bar{\boldsymbol{x}}_{k}^{\boldsymbol{w}})\Big\|^{2}.
	\end{align*}
	Thus, $\big\langle \nabla^{2} f(\boldsymbol{w}) \boldsymbol{\xi}, \boldsymbol{\xi}\big\rangle_{\mathcal{F}} \leq 0$.	
\end{proof}

\subsection{Proof of Remark~\ref{REMARK:CURVATURE_CONSTANT}}

\begin{proof}
    We compute
    \begin{align*}
        \Big|\big\langle \nabla^{2} f(\boldsymbol{w}) \boldsymbol{\xi}, \boldsymbol{\xi}\big\rangle_{\mathcal{F}}\Big| &\leq
        2 \sum_{k = 1}^{K} \frac{1}{n_{k}^{\boldsymbol{w}}} 
        \Big(\max_{i = 1, \dots, n \atop k = 1, \dots, K} \|\boldsymbol{x}_{i} - \bar{\boldsymbol{x}}_{k}^{\boldsymbol{w}}\|\Big)^{2} \sum_{i = 1}^{n} \xi_{ik}^{2} \\
        &
        \leq 2 \big(\max_{i, j = 1, \dots, n} \|\boldsymbol{x}_{i} - \boldsymbol{x}_{j}\|^{2}\big) \|\boldsymbol{\xi} \|_{\mathcal{F}}^{2} \\
        &\leq 2 \|\boldsymbol{\xi}\|_{\mathcal{F}}^{2}
        \max_{i = 1, \dots, n \atop k = 1, \dots, K} \|\boldsymbol{x}_{i} - \bar{\boldsymbol{x}}_{k}^{\boldsymbol{w}}\|^{2} \\
        &\leq 2 \max_{i = 1, \dots, n \atop k = 1, \dots, K} \|\boldsymbol{x}_{i} - \bar{\boldsymbol{x}}_{k}^{\boldsymbol{w}}\|^{2}
        \leq 2 \Big(\max_{i, j = 1, \dots, n} \|\boldsymbol{x}_{i} - \boldsymbol{x}_{j}\|\Big)^{2}
    \end{align*}
    for all $\boldsymbol{\xi} \in \mathbb{R}^{n \times K}$, which estimates the operator norm of $\nabla^{2} f$.
\end{proof}

\newpage

\section{\MakeUppercase{Additional Figures}}

\label{APPENDIX:ADDITIONAL_FIGURES}

\foreach \indexd in {2, 5, 10} {%
    \subsection{Dimension $d = \indexd$}
    
    \foreach \indexn in {500, 1000, 5000} {%
        \begin{figure}[h!]
            \centering
            \begin{minipage}[t]{0.8\textwidth}
                \includegraphics[width=1.0\linewidth]{fig/plot.n=\indexn.d=\indexd.pdf}
            \end{minipage}
            \caption{Worst-case FW convergence rate on synthetic ``blobs'' data for $n = \indexn$ and $d = \indexd$.}
        \end{figure}
    }
    \clearpage
}

\newpage

\section{\MakeUppercase{Tabulated Simulation Results}}

\label{APPENDIX:TABLES}

The $p$-values reported below are from Student's $t$-test of $H_{0}: \beta_{1} = -1$ vs $H_{1}: \beta_{1} > -1$.

\begin{table}[h!]
\centering
\caption{Simulation results for $d = 2$.}
\vspace{0.05in}

\begin{tabular}{ccccccc}
& & \multicolumn{2}{c}{Theoretical bound} & \multicolumn{2}{c}{Empirical fit} \\
$n$ & $K$ & Slope & Intercept & Slope (SE) & Intercept (SE) & $p$-value \\
\midrule
500  & 5  & -1.0000 & 11.1017 & -0.8559 (0.1273) & 10.6940 (0.2399) & 0.1409 \\
500  & 10 & -1.0000 & 10.9562 & -1.1672 (0.0982) & 10.0712 (0.1971) & 0.9438 \\
500  & 20 & -1.0000 & 10.8755 & -1.3118 (0.1165) &  8.8922 (0.2268) & 0.9899 \\
500  & 50 & -1.0000 & 10.7266 & -1.7189 (0.1005) &  7.6737 (0.1584) & 0.9999 \\
1000 & 5  & -1.0000 & 11.7480 & -1.2649 (0.1317) & 12.0234 (0.2717) & 0.9681 \\
1000 & 10 & -1.0000 & 11.6464 & -1.2990 (0.1406) & 10.8671 (0.2902) & 0.9741 \\
1000 & 20 & -1.0000 & 11.5446 & -1.2337 (0.1407) &  9.6810 (0.2978) & 0.9413 \\
1000 & 50 & -1.0000 & 11.4175 & -1.4563 (0.0867) &  8.0852 (0.1689) & 0.9999 \\
5000 & 5  & -1.0000 & 13.4988 & -0.8214 (0.1573) & 12.9955 (0.3623) & 0.1351 \\
5000 & 10 & -1.0000 & 13.3859 & -0.9913 (0.0972) & 12.0877 (0.2354) & 0.4645 \\
5000 & 20 & -1.0000 & 13.1775 & -1.2177 (0.0889) & 11.1182 (0.2250) & 0.9891 \\
5000 & 50 & -1.0000 & 13.0158 & -1.3904 (0.0785) &  9.6238 (0.1960) & 1.0000 \\
\end{tabular}
\end{table}

\begin{table}[h!]
\centering
\caption{Simulation results for $d = 5$.}
\vspace{0.05in}

\begin{tabular}{ccccccc}
& & \multicolumn{2}{c}{Theoretical bound} & \multicolumn{2}{c}{Empirical fit} \\
$n$ & $K$ & Slope & Intercept & Slope (SE) & Intercept (SE) & $p$-value \\
\midrule
500  & 5  & -1.0000 & 11.8339 & -1.0818 (0.1289) & 11.6820 (0.2031) & 0.7272 \\
500  & 10 & -1.0000 & 11.7078 & -1.1657 (0.1463) & 10.9860 (0.2009) & 0.8455 \\
500  & 20 & -1.0000 & 11.6158 & -1.1289 (0.1609) & 10.2478 (0.2676) & 0.7769 \\
500  & 50 & -1.0000 & 11.5441 & -1.4664 (0.0513) &  9.4392 (0.0705) & 0.9999 \\
1000 & 5  & -1.0000 & 12.4600 & -0.7593 (0.1512) & 12.0811 (0.2634) & 0.0729 \\
1000 & 10 & -1.0000 & 12.3841 & -1.0082 (0.1402) & 11.5376 (0.2332) & 0.5227 \\
1000 & 20 & -1.0000 & 12.3113 & -1.0941 (0.0983) & 10.8865 (0.1635) & 0.8167 \\
1000 & 50 & -1.0000 & 12.2349 & -1.3659 (0.0411) & 10.1017 (0.0716) & 1.0000 \\
5000 & 5  & -1.0000 & 14.0585 & -0.7666 (0.1328) & 13.7828 (0.2666) & 0.0512 \\
5000 & 10 & -1.0000 & 13.9688 & -0.8322 (0.0926) & 12.9732 (0.1804) & 0.0476 \\
5000 & 20 & -1.0000 & 13.9161 & -0.9369 (0.1088) & 12.3498 (0.2120) & 0.2865 \\
5000 & 50 & -1.0000 & 13.8460 & -1.1847 (0.0780) & 11.6294 (0.1470) & 0.9814 \\
\end{tabular}
\end{table}

\begin{table}[h!]
\centering
\caption{Simulation results for $d = 10$.}
\vspace{0.05in}

\begin{tabular}{ccccccc}
& & \multicolumn{2}{c}{Theoretical bound} & \multicolumn{2}{c}{Empirical fit} \\
$n$ & $K$ & Slope & Intercept & Slope (SE) & Intercept (SE) & $p$-value \\
\midrule
500  & 5  & -1.0000 & 12.3848 & -1.1529 (0.1883) & 12.3593 (0.2787) & 0.7761 \\
500  & 10 & -1.0000 & 12.2919 & -1.1001 (0.1684) & 11.5409 (0.2313) & 0.7109 \\
500  & 20 & -1.0000 & 12.2521 & -1.4784 (0.0751) & 11.3285 (0.0940) & 0.9984 \\
500  & 50 & -1.0000 & 12.1881 & -1.6946 (0.0637) & 10.6826 (0.0798) & 0.9998 \\
1000 & 5  & -1.0000 & 13.0277 & -1.3164 (0.2652) & 13.1580 (0.4815) & 0.8698 \\
1000 & 10 & -1.0000 & 12.9742 & -1.0745 (0.1584) & 12.2400 (0.2344) & 0.6727 \\
1000 & 20 & -1.0000 & 12.9085 & -1.3027 (0.1303) & 11.7727 (0.1631) & 0.9596 \\
1000 & 50 & -1.0000 & 12.8820 & -1.7510 (0.0650) & 11.4515 (0.0814) & 0.9998 \\
5000 & 5  & -1.0000 & 14.6173 & -0.9409 (0.1161) & 14.4244 (0.2108) & 0.3109 \\
5000 & 10 & -1.0000 & 14.5702 & -0.9001 (0.1097) & 13.7718 (0.2067) & 0.1909 \\
5000 & 20 & -1.0000 & 14.5348 & -1.0433 (0.1284) & 13.2512 (0.2238) & 0.6280 \\
5000 & 50 & -1.0000 & 14.4765 & -1.4305 (0.0936) & 12.9385 (0.1475) & 0.9988 \\
\end{tabular}
\end{table}

%%%%%%%%%%%%%%%%%%%%%%%%%%%%%%%%%%%%%%%%%%%%%%%%%%%%%%%%%%%%

% Note: You can choose whether the include you appendices as part of the main submission file (here) OR submit them separately as part of the supplementary material. It is the authors' responsibility that any supplementary material does not conflict in content with the main paper (e.g., the separately uploaded additional material is not an updated version of the one appended to the manuscript).

% \section{FORMATTING INSTRUCTIONS}

% The appendices follow the same formatting instructions as in the main paper.
% The only difference is that the supplementary material must be in a \emph{single-column} format.

% Note that reviewers are under no obligation to examine your supplementary material.

% \section{MISSING PROOFS}

% The supplementary materials may contain detailed proofs of the results that are missing in the main paper.

% \subsection{Proof of Lemma 3}

% \textit{In this section, we present the detailed proof of Lemma 3 and then [ ... ]}

% \section{ADDITIONAL EXPERIMENTS}

% If you have additional experimental results, you may include them in the supplementary materials.

% \subsection{Effect of the Regularization Parameter}

% \textit{Our algorithm depends on the regularization parameter $\lambda$. Figure 1 below illustrates the effect of this parameter on the performance of our algorithm. As we can see, [ ... ]}

\end{document}